\documentclass[a4paper]{article}

\usepackage{amsmath}
\usepackage{amsthm}
\usepackage{amsfonts}
\usepackage{graphicx}
\usepackage{booktabs}
\usepackage{footnote}
\makesavenoteenv{tabular}
\makesavenoteenv{table}

\newtheorem{prop}{Proposition}

\newcommand{\vectorfmt}[1]{\mathbf{#1}}

\begin{document}

\author{Katarzyna Janocha$^1$, Wojciech Marian Czarnecki$^{2,1}$ \\
\small $^1$Faculty of Mathematics and Computer Science,\\
\small Jagiellonian University, Krakow, Poland \\
\small $^2$DeepMind, London, UK\\
\small e-mail: {\it kasiajanocha@gmail.com,  lejlot@google.com}
}
\date{}
%\author{Katarzyna Janocha$^1$, Wojciech Marian Czarnecki$^{2,1}$ \\
%{\small\rm  $^1$Faculty of Mathematics and Computer Science,\\
%Jagiellonian University, Krakow, Poland \\
% $^2$DeepMind, London, UK \\
%e-mail: {\it kasiajanocha@gmail.com,  lejlot@google.com}}}

\title{{\large\bf On Loss Functions for Deep Neural Networks\\ in Classification}}

\maketitle

\abstract{
Deep neural networks are currently among the most commonly used classifiers.
Despite easily achieving very good performance, one of the best selling points of these models
is their modular design -- one can conveniently adapt their architecture to specific needs, change connectivity patterns,
attach specialised layers, experiment with a large amount of activation functions,
normalisation schemes and many others. While one can find impressively wide spread of various
configurations of almost every aspect of the deep nets, one element is, in authors' opinion, underrepresented -- 
while solving classification problems, vast majority of papers and applications simply use log loss.
%Very few papers actually analyze the impact of the loss functions on the training behaviour as well as the resulting model.
In this paper we try to investigate how particular choices of loss functions affect deep models and their learning dynamics,
as well as resulting classifiers robustness to various effects. We perform experiments on classical datasets, as well
as provide some additional, theoretical insights into the problem. In particular we show that $\mathcal{L}_1$ and $\mathcal{L}_2$ losses
are, quite surprisingly, justified classification objectives for deep nets, by providing probabilistic interpretation in terms of expected misclassification. We also introduce two losses which are not typically used as deep nets objectives and show that they are viable 
alternatives to the existing ones.
}

%\keywords{loss function, deep learning, classification theory.} 

\section{Introduction}

For the last few years the Deep Learning (DL) research has been rapidly developing.
It evolved from tricky pretraining routines~\cite{larochelle2009exploring} to a highly modular, customisable framework for building
machine learning systems for various problems, spanning from image recognition~\cite{krizhevsky2012imagenet}, voice recognition and
synthesis~\cite{oord2016wavenet} to complex AI systems~\cite{silver2016mastering}. One of the biggest advantages of DL is enormous flexibility in designing
each part of the architecture, resulting in numerous ways of putting priors over data inside the model itself~\cite{larochelle2009exploring},
finding the most efficient activation functions~\cite{clevert2015fast} or learning algorithms~\cite{kingma2014adam}.
%or signal normalization~\cite{ioffe2015batch}.
However, to authors' best knowledge, most of the community still keeps one element nearly completely fixed -- when it comes
to classification, we use log loss (applied to softmax activation of the output of the network). In this paper we try to address
this issue by performing both theoretical and empirical analysis of effects various loss functions have on the training of deep nets.

It is worth noting that Tang et al.~\cite{tang2013deep} showed that well fitted hinge loss
can outperform log loss based networks in typical classification tasks. Lee et al.~\cite{lee2015deeply} used squared hinge loss for classification
tasks, achieving very good results. From slightly more theoretical perspective Choromanska et al.~\cite{choromanska2015loss} also considered $\mathcal{L}_1$ loss as a deep net objective.
However, these works seem to be exceptions, appear in complete separation from one another, and usually do not focus on any effect of the loss function but the final performance.
Our goal is to show these losses in a wider context, comparing one another under various criteria and provide insights into when -- and why -- one should use them.

\begin{table}[h]
\centering \caption{List of losses analysed in this paper. $\mathbf{y}$ is true label as one-hot encoding, $\mathbf{\hat y}$ is true label as +1/-1 encoding, $\mathbf{o}$ is the output of the last layer of the network, $\cdot^{(j)}$ denotes $j$th dimension of a given vector, and $\sigma(\cdot)$ denotes probability estimate.}

 \begin{tabular}{lll}
 \toprule
 symbol & name & equation \\
 \midrule
 $\mathcal{L}_1$ & \small L$_1$ loss & $\|\mathbf{y} - \mathbf{o}\|_1$\\
 $\mathcal{L}_2$ &  \small L$_2$ loss & $\|\mathbf{y} - \mathbf{o}\|_2^2$\\
 $\mathcal{L}_1 \circ \sigma$ &  \small expectation loss & $\|\mathbf{y} - \sigma(\textbf{o})\|_1$\\
 $\mathcal{L}_2 \circ \sigma$ &  \small regularised expectation loss\footnote{See Proposition~\ref{prop:l1exp}}& $\|\mathbf{y} - \sigma(\textbf{o})\|_2^2$\\
 $\mathcal{L}_\infty \circ \sigma$ &  \small Chebyshev loss & $\max_j |\sigma(\mathbf{o})^{(j)} - \mathbf{y}^{(j)}|$ \\
 hinge &  \small hinge~\cite{tang2013deep} (margin) loss& $\sum_{j} \max(0, \tfrac{1}{2} - \mathbf{\hat y}^{(j)}\mathbf{o}^{(j)}) $\\
 hinge$^2$ &  \small squared hinge (margin) loss& $\sum_{j} \max(0, \tfrac{1}{2} - \mathbf{\hat y}^{(j)}\mathbf{o}^{(j)})^2 $\\
 hinge$^3$ &  \small cubed hinge (margin) loss& $\sum_{j} \max(0,\tfrac{1}{2} - \mathbf{\hat y}^{(j)}\mathbf{o}^{(j)})^3 $\\
 log &  \small log (cross entropy) loss & $-\sum_{j}  \mathbf{y}^{(j)}\log \sigma(\mathbf{o})^{(j)} $\\
 log$^2$ &  \small squared log loss & $- \sum_{j} [\mathbf{y}^{(j)}\log \sigma(\mathbf{o})^{(j)}]^2 $\\
 tan &  \small Tanimoto loss & $\tfrac{-\sum_j \sigma(\mathbf{o})^{(j)} \mathbf{y}^{(j)}}{\|\sigma(\mathbf{o})\|_2^2+\|\mathbf{y}\|_2^2-\sum_j\sigma(\mathbf{o})^{(j)} \mathbf{y}^{(j)}}$ \\
 D$_\mathrm{CS}$ &  \small Cauchy-Schwarz Divergence~\cite{czarnecki2015maximum} & 
 $-\log \tfrac{\sum_j \sigma(\mathbf{o})^{(j)} \mathbf{y}^{(j)}}{\|\sigma(\mathbf{o})\|_2 \|\mathbf{y}\|_2}$ \\
 \bottomrule
 \end{tabular}
 \label{tab:losses}
\end{table}
This work focuses on 12 loss functions, described in Table~\ref{tab:losses}.
Most of them appear in deep learning (or more generally -- machine learning) literature, however some in slightly different context than a classification loss.
In the following section we present new insights into theoretical properties of a couple of these losses and then provide experimental evaluation of resulting models' properties, including the effect on speed of learning, final performance, input data and label noise robustness as well as convergence for simple dataset under limited resources regime.

\section{Theory}
% 
% $$
%  \ell_\mathrm{exp}(\vectorfmt{p}, \vectorfmt{y}) =
%   \sum\nolimits_j \left [ (1-\vectorfmt{p}_j) \vectorfmt{y}_j + \vectorfmt{p}_j(1-\vectorfmt{y}_j) \right ]
% $$

Let us begin with showing interesting properties of $\mathcal{L}_p$ functions, typically considered as purely regressive losses, which should not be used in classification.
$\mathcal{L}_1$ is often used as an auxiliary loss in deep nets to ensure sparseness of representations. Similarly, $\mathcal{L}_2$ is sometimes (however nowadays quite rarely)
applied to weights in order to prevent them from growing to infinity.
In this section we show that -- despite their regression roots  -- they still have reasonable probabilistic interpretation for classification and can be used as a main classification objective.

We use the following notation: $\{(\mathbf{x}_i,\mathbf{y}_i)\}_{i=1}^N \subset \mathbb{R}^d \times \{0,1\}^K$ is a training set, an iid sample from unknown $P(\mathbf{x},\mathbf{y})$ and $\sigma$ denotes a function producing probability estimates (usually sigmoid or softmax).

\begin{prop}
\label{prop:l1exp}
$\mathcal{L}_1$ loss applied to the probability estimates $\hat p(\vectorfmt{y}|\vectorfmt{x})$ leads to minimisation of expected misclassification probability (as opposed to maximisation of fully correct labelling given by the log loss). Similarly $\mathcal{L}_2$ minimises the same factor, but regularised with a half of expected squared L$_2$ norm of the predictions probability estimates. 
\end{prop}
\begin{proof}
 In $K$-class classification dependent variables are vectors $\vectorfmt{y}_i \in \{0,1\}^K$ with $\mathrm{L}_1(\vectorfmt{y}_i)=1$, thus using notation 
 $\vectorfmt{p}_i = \hat p(\vectorfmt{y}|\vectorfmt{x}_i)$
 \begin{equation*}
  \begin{aligned}
\mathcal{L}_1
 &= \tfrac{1}{N}\sum\nolimits_i \sum\nolimits_j | {\vectorfmt{p}_i^{(j)}} - {\vectorfmt{y}_i^{(j)}} |
 = \tfrac{1}{N}\sum\nolimits_i \left [ \sum\nolimits_j {\vectorfmt{y}_i^{(j)}}(1-{\vectorfmt{p}_i^{(j)}}) + (1-{\vectorfmt{y}_i^{(j)}}){\vectorfmt{p}_i^{(j)}} \right ]\\
 &= \tfrac{1}{N}\sum\nolimits_i \left [ \sum\nolimits_j {\vectorfmt{y}_i^{(j)}}-2\sum\nolimits_j{\vectorfmt{y}_i^{(j)}}{\vectorfmt{p}_i^{(j)}} + \sum\nolimits_j{\vectorfmt{p}_i^{(j)}} \right ]
 =  2 -2\tfrac{1}{N}\sum\nolimits_i \left [\sum\nolimits_j{\vectorfmt{y}_i^{(j)}}{\vectorfmt{p}_i^{(j)}} \right ].
 \end{aligned}
 \end{equation*}
 Consequently if we sample label according to $\vectorfmt{p}_i$ then probability that it actually matches one hot encoded label in $\vectorfmt{y}_i$ equals
 $
 P(\hat l = l | \hat l \sim \vectorfmt{p}_i, l \sim \vectorfmt{y}_i) 
 = \sum\nolimits_j {\vectorfmt{y}_i^{(j)}} {\vectorfmt{p}_i^{(j)}},
 $
 and consequently
 \begin{equation*}
  \begin{aligned}   
 \mathcal{L}_1 &= 2 -2\tfrac{1}{N}\sum\nolimits_i \left [\sum\nolimits_j{\vectorfmt{y}_i^{(j)}}{\vectorfmt{p}_i^{(j)}} \right ]
 \approx - 2\mathbb{E}_{P(\vectorfmt{x},\vectorfmt{y})}\left [ P(\hat l = l | \hat l \sim \vectorfmt{p}_i, l \sim \vectorfmt{y}_i) \right ] + \text{const.}
  \end{aligned}
 \end{equation*}
Analogously for $\mathcal{L}_2,$
 \begin{equation*}
  \begin{aligned}   
\mathcal{L}_2
&=  - 2\tfrac{1}{N}\sum\nolimits_i \left [\sum\nolimits_j{\vectorfmt{y}_i^{(j)}}{\vectorfmt{p}_i^{(j)}} \right ]
 + \tfrac{1}{N}\sum\nolimits_i \mathrm{L}_2 ({\vectorfmt{y}_i} )^2
  + \tfrac{1}{N}\sum\nolimits_i \mathrm{L}_2 ({\vectorfmt{p}_i} )^2\\
 & \approx - 2\mathbb{E}_{P(\vectorfmt{x},\vectorfmt{y})}\left [ P(\hat l = l | \hat l \sim \vectorfmt{p}_i, l \sim \vectorfmt{y}_i) \right ]
 + \mathbb{E}_{P(\vectorfmt{x},\vectorfmt{y})}[\mathrm{L}_2 ({\vectorfmt{p}_i} )^2] + \text{const.}
  \end{aligned}
 \end{equation*}
\end{proof}
For this reason we  refer to these losses as \emph{expectation loss} and  \emph{regularised expectation loss} respectively.
One could expect that this should lead to higher robustness to the outliers/noise, as we try to maximise the expected probability of good classification as opposed to the probability of completely correct labelling (which log loss does). Indeed, as we show in the experimental section -- this property is true for all losses sharing connection with \emph{expectation losses}.
% since its simpler definition is
% $$
% \mathcal{L}_\text{exp} = -\tfrac{1}{N} \sum\nolimits_i \sum\nolimits_j {\mathbf{y}_i^{(j)}} {\mathbf{p}_i^{(j)}} = -\mathbb{E}_{P(\vectorfmt{x},\vectorfmt{y})}\left [ P(\hat l = l | \hat l \sim \vectorfmt{p}_i, l \sim \vectorfmt{y}_i) \right ].
% 
% $$

% It seems natural to ask whether similar probabilistic interpretation holds for $\mathcal{L}_2$, which is still sometimes used as a loss for classification problems. Following proposition addresses this question
% 
% \begin{prop}
% \label{prop:l2exp}
% $\mathcal{L}_2$ loss applied to the probability estimates $\hat p(\vectorfmt{y}|\vectorfmt{x})$ leads to minimization of expected missclassification probability regularized with the half of expected squared $l_2$ norm of the predictions probability estimates. 
% \end{prop}
% \begin{proof}
% Similarly to previous Proposition we compute
% \begin{equation*}
% \begin{aligned} 
% \mathcal{L}_2
% &= 
%  =  - 2\tfrac{1}{N}\sum\nolimits_i \left [\sum\nolimits_j{\vectorfmt{y}_i^{(j)}}{\vectorfmt{p}_i^{(j)}} \right ]
%  + \tfrac{1}{N}\sum\nolimits_i l_2 ({\vectorfmt{y}_i} )^2
%   + \tfrac{1}{N}\sum\nolimits_i l_2 ({\vectorfmt{p}_i} )^2\\
%  & \approx - 2\mathbb{E}_{P(\vectorfmt{x},\vectorfmt{y})}\left [ P(\hat l = l | \hat l \sim \vectorfmt{p}_i, l \sim \vectorfmt{y}_i) \right ]
%  + \mathbb{E}_{P(\vectorfmt{x},\vectorfmt{y})}[l_2 ({\vectorfmt{p}_i} )^2] + \text{const.}
% \end{aligned}
% \end{equation*}
% \end{proof}
% \noindent Consequently we will refer to this loss as

So why is using these two loss functions unpopular? Is there anything fundamentally wrong with this formulation from the mathematical perspective?
While the following observation is not definitive, it shows an insight into what might be the issue causing slow convergence of such methods.

\begin{prop}
\label{prop:vanish}
 $\mathcal{L}_1$, $\mathcal{L}_2$ losses applied to probabilities estimates coming from sigmoid (or softmax) 
%  as well as Tanimoto  loss 
have non-monotonic partial derivatives wrt. to the output of the final layer (and the loss is not convex nor concave wrt. to last layer weights). Furthermore, they vanish in both infinities, which slows down learning of heavily misclassified examples. 
\end{prop}
\begin{proof}
Let us denote sigmoid activation as $\sigma(x) = (1+e^{-x})^{-1}$ and, without loss of generality, compute partial derivative of $\mathcal{L}_1$ when network is presented with $x_p$ with positive label. Let $o_p$ denote the output activation for this sample.
 \begin{equation*} 
 \begin{aligned}
 \frac{\partial (\mathcal{L}_1 \circ  \sigma)}{\partial o}(o_p) = \frac{\partial}{\partial o} \left ( | 1 - (1+e^{-o})^{-1} | \right )(o_p) = 
   -\frac{e^{-o_p}}{(e^{-o_p}+1)^2}\\
 \lim_{o \rightarrow -\infty} -\frac{e^{-o}}{(e^{-o}+1)^2} = 0 = \lim_{o \rightarrow \infty} -\frac{e^{-o}}{(e^{-o}+1)^2},
 \end{aligned}
\end{equation*}
 while at the same time $-\frac{e^{0}}{(e^{0}+1)^2} = -\tfrac{1}{4} < 0$, completing the proof of both non-monotonicity as well as the fact it vanishes when point is heavily misclassified. Lack of convexity comes from the same argument since second derivative wrt. to any weight in the final layer of the model changes sign (as it is equivalent to first derivative being non-monotonic). This comes directly from the above computations and the fact that $o_p = \langle \mathbf{w}, \mathbf{h}_p \rangle + b $ for some internal activation $\mathbf{h}_p$, layer weights $\mathbf{w}$ and layer bias $b$. In a natural way this is true even if we do not have any hidden layers (model is linear). Proofs for $\mathcal{L}_2$ and softmax are completely analogous. 
\end{proof}

Given this negative result, it seems natural to ask whether a similar property can be proven to show which loss functions should lead to \emph{fast} convergence.
It seems like the answer is again positive, however based on the well known deep learning hypothesis that deep models learn well when dealing with piece-wise linear functions. 
%Of course this is not really deep learning specific, optimization of piecewise linear functions is simply one of the most basic types of optimization. 
An interesting phenomenon in classification based on neural networks is that even in a deep linear model or rectifier network the top layer is often non-linear, as it uses softmax or
sigmoid activation to produce probability estimates. Once this is introduced, also the partial derivatives stop being piece-wise linear. We believe that one can achieve faster, better
convergence when we ensure that architecture together with loss function, produces a piecewise linear partial derivatives (but not constant) wrt. to final layer activations,
especially while using first order optimisation methods. This property is true only for $\mathcal{L}_2$ loss and squared hinge loss (see Figure~\ref{fig:deriviatives}) among all considered ones in this paper.
\begin{figure}[h]
 \includegraphics[width=0.375\textwidth]{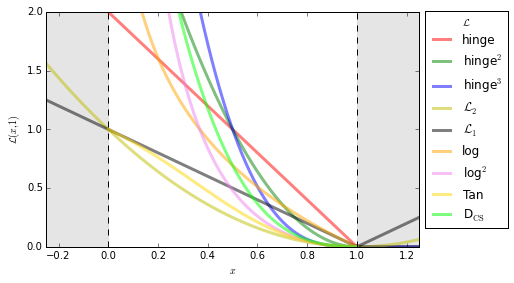}
 \includegraphics[width=0.3\textwidth]{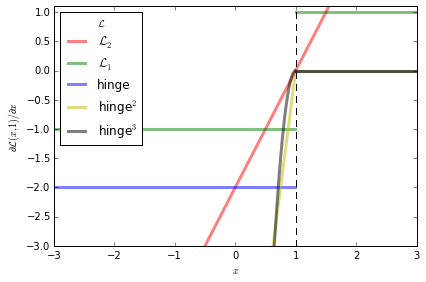}
 \includegraphics[width=0.3\textwidth]{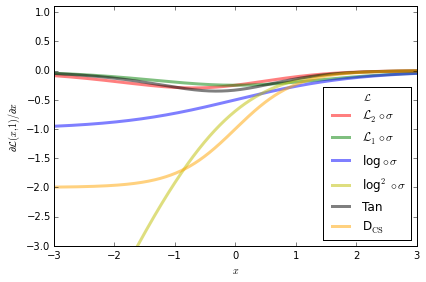}
\caption{Left: Visualisation of analysed losses as functions of activation on positive sample. Middle: Visualisation of partial derivatives wrt. to output neuron for losses based on linear output. Right: Visualisation of partial derivatives wrt. to output neuron for losses based on probability estimates.}
\label{fig:deriviatives}
\end{figure}
% \begin{prop}
% $\mathcal{L}_2$ loss (applied directly to the output of the network) as well as squared hinge loss produce piecewise-linear partial derivatives wrt. to activations of the final layer with at least one non-constant part. At the same time this is not true for the remaining analyzed losses.
% \end{prop}
% \begin{proof}
% Comes directly from the equations of the losses and definition of the partial derivative. See also Figure~\ref{fig:deriviatives}.
% \end{proof}

Finally we show relation between Cauchy-Schwarz Divergence loss and the log loss, justifying its introduction as an objective for neural nets.
\begin{prop}
Cauchy-Schwarz Divergence loss is equivalent to cross entropy loss regularised with half of expected Renyi's quadratic entropy of the predictions.
\end{prop}
\begin{proof}
Using the fact that $\forall_i\exists!_j : {\mathbf{y}_i^{(j)}} = 1$ we get that $ \log \sum\nolimits_j {\mathbf{p}_i^{(j)}} {\mathbf{y}_i^{(j)}} = \sum\nolimits_j  {\mathbf{y}_i^{(j)}} \log {\mathbf{p}_i^{(j)}} $ as well as $\|\mathbf{y}_i\|_2 = 1$, consequently 
\begin{equation*}
\begin{aligned}
D_\mathrm{CS} & = - \tfrac{1}{N} \sum\nolimits_i \log \tfrac{\sum\nolimits_j {\mathbf{p}_i^{(j)}} {\mathbf{y}_i^{(j)}}}{\| {\mathbf{p}_i} \|_2 \| {\mathbf{y}_i} \|_2} =
- \tfrac{1}{N} \sum\nolimits_i  \log \sum\nolimits_j {\mathbf{p}_i^{(j)}} {\mathbf{y}_i^{(j)}} + \tfrac{1}{N} \sum\nolimits_i \log \| {\mathbf{p}_i} \|_2 \| {\mathbf{y}_i} \|_2 \\
&=- \tfrac{1}{N} \sum\nolimits_i \sum\nolimits_j {\mathbf{y}_i^{(j)}} \log {\mathbf{p}_i^{(j)}}  + \tfrac{1}{2N} \sum\nolimits_i \log \| {\mathbf{p}_i} \|^2_2 \approx 
\mathcal{L}_\mathrm{log} + \tfrac{1}{2}\mathbb{E}_{P(\mathbf{x},\mathbf{y})}[H_2(\mathbf{p}_i)]
\end{aligned}
\end{equation*}
\end{proof}

\section{Experiments}

We begin the experimental section with two simple 2D toy datasets. The first one is checkerboard -- 4 class classification problem where [-1,1] square is divided into 64 small squares with cyclic class assignment.
The second one, spiral, is a 4 class generalisation of the well known 2 spirals dataset. Both datasets have 800 training and 800 testing samples. We train rectifier neural network having from 0 to 5 hidden layers with 200 units in each of them. Training is performed using Adam~\cite{kingma2014adam} with learning rate of $0.00003$ for 60,000 iterations with batch size of 50 samples.
\begin{figure}[h]
\centering
 \includegraphics[width=\textwidth]{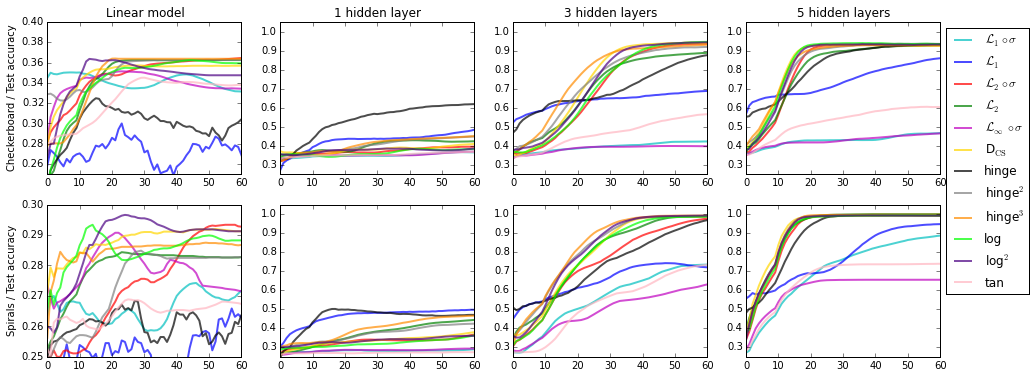}
 \includegraphics[width=0.23\textwidth]{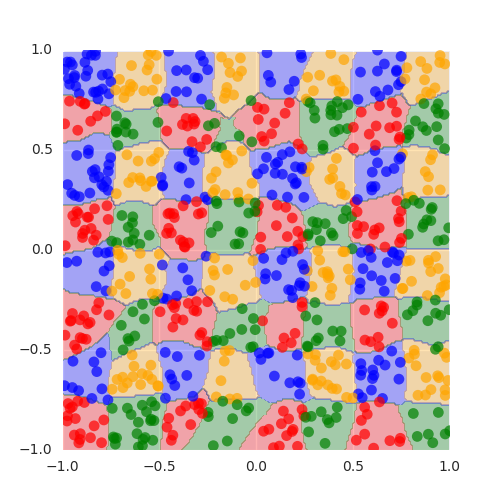}
 \includegraphics[width=0.23\textwidth]{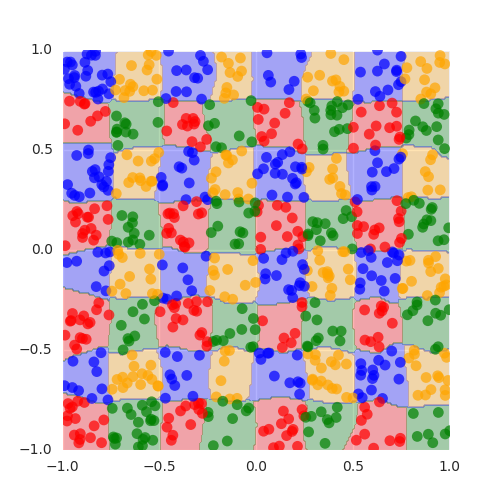}
 \includegraphics[width=0.23\textwidth]{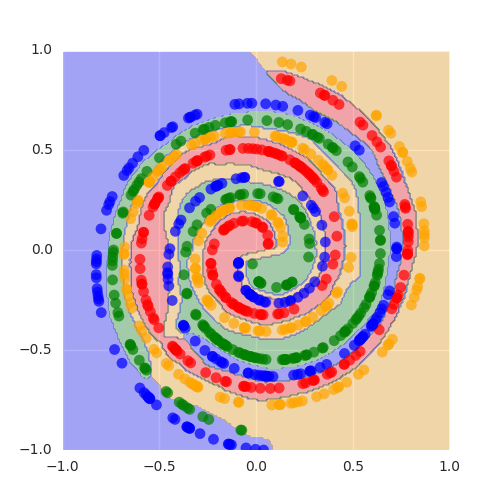}
 \includegraphics[width=0.23\textwidth]{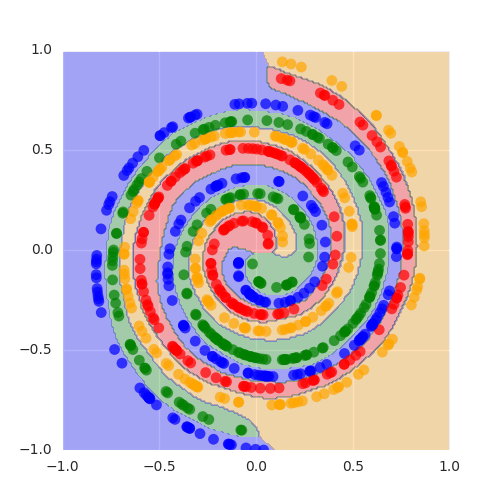}
 \caption{Top row: Learning curves for toy datasets. Bottom row: examples of decision boundaries, from left: $\mathcal{L}_1$ loss, log loss, $\mathcal{L}_1 \circ \sigma$ loss, hinge$^2$ loss.}
 \label{fig:toy_train}
\end{figure}
In these simple problems one can distinguish (Figure~\ref{fig:toy_train}) two groups of losses -- one able to fit to our very dense, low-dimensional data and one struggling 
to reduce error to 0. The second group consists of $\mathcal{L}_1$, Chebyshev, Tanimoto and expectation loss. This division becomes clear once we build a relatively
deep model (5 hidden layers), while for shallow ones this distinction is not very clear (3 hidden layers) or is even completely lost (1 hidden layer or linear model).
To further confirm the lack of ability to easily overfit we also ran an experiment in which we tried to fit 800 samples from uniform distribution over $[-1,1]$ with randomly assigned 4 labels and achieved analogous partitioning.

During following, real data-based experiments, we focus on further investigation of loss functions properties emerging after application to deep models, as well as characteristics of the created models.
In particular, we show that lack of ability to reduce training error to 0 is often correlated with robustness to various types of noise (despite not underfitting the data).

% \section{Learning speed and generalization strength}
Let us now proceed with one of the most common datasets used in deep learning community -- MNIST~\cite{lecun1998mnist}. We train network consisting from 0 to 5 hidden layers, each followed by ReLU activation function and dropout~\cite{srivastava2014dropout} with 50\% probability. Each hidden layer consists of 512 neurons, and whole model is trained using Adam~\cite{kingma2014adam} with learning rate of $0.00003$ for 100,000 iterations using batch size of 100 samples. 
\begin{figure}[h]
\centering
 \includegraphics[width=\textwidth]{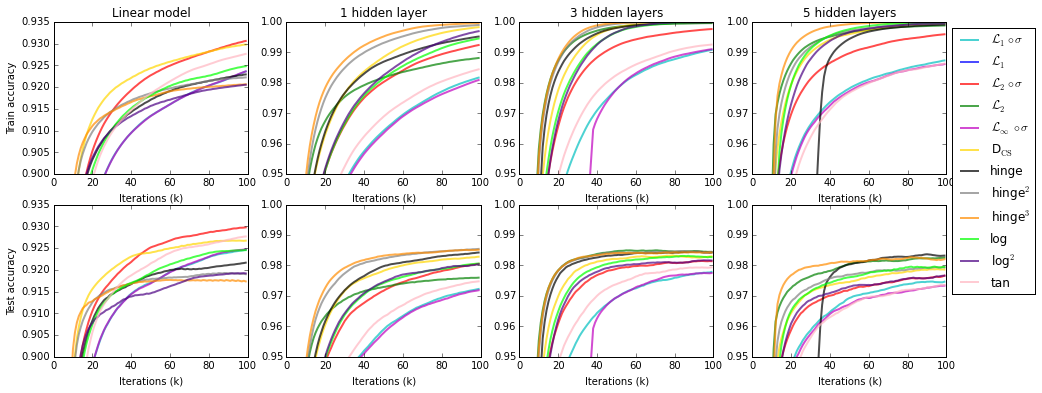}\vspace{0.25cm}
 \includegraphics[width=0.25\textwidth]{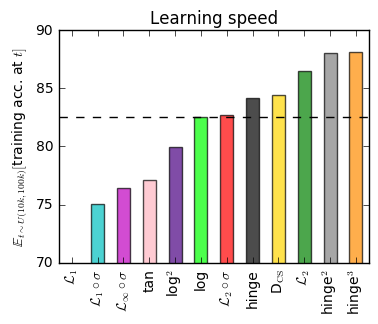}
 \includegraphics[width=0.25\textwidth]{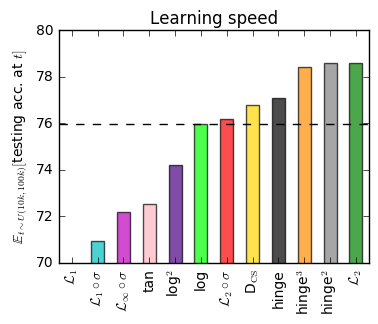}
 \includegraphics[width=0.47\textwidth]{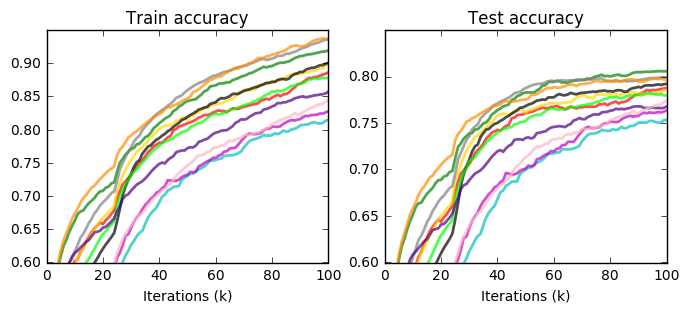}
 \caption{Top two rows: learning curves for MNIST dataset. Bottom row: (left) speed of learning expressed as expected training/testing accuracy when we sample iteration uniformly between 10k and 100k; (right) learning curves for CIFAR10 dataset.}
 \label{fig:mnist_train}
\end{figure}
There are few interesting findings, visible on Figure~\ref{fig:mnist_train}. First, results obtained for a linear model (lack of hidden layers) are qualitatively different from all the remaining ones. For example, using regularised expectation loss leads to the strongest model in terms of both training accuracy and generalisation capabilities, while the same loss function is far from being the best one once we introduce non-linearities.
This shows two important things: first -- observations and conclusions drawn from linear models do not seem to transfer to deep nets, and second -- there seems to be an interesting co-dependence between learning dynamics coming from training rectifier nets and loss functions used. As a side note, 93\% testing accuracy, obtained by $\mathcal{L}_2 \circ \sigma$ and $D_\mathrm{CS}$, is a very strong result on MNIST using linear model without any data augmentation or model regularisation.

Second interesting observation regards the speed of learning. It appears that (apart from linear models) hinge$^2$ and hinge$^3$ losses are consistently the fastest in training, and once we have enough hidden layers (basically more than 1) also $\mathcal{L}_2$. This matches our theoretical analysis of these losses in the previous section. At the same time both expectation losses are much slower to train, which we believe to be a result of their vanishing partial derivatives in heavily misclassified points (Proposition~\ref{prop:vanish}).
It is important to notice that while higher order hinge losses (especially 2$^\mathrm{nd}$) actually help in terms of both speed and final performance, the same property does not hold for higher order log losses. One possible explanation is that taking a square of log loss only reduces model's certainty in classification (since any number between 0 and 1 taken to 2$^\mathrm{nd}$ power decreases), while for hinge losses the metric used for penalising margin-outliers is changed, and both L$_1$ metric (leading to hinge) as well as any other L$_p$ norm (leading to hinge$^p$) make perfect sense.
 
Third remark is that pure $\mathcal{L}_1$ does not learn at all (ending up with ~20\% accuracy) due to causing serious ``jumps'' in the model because of its partial derivatives wrt.
to net output always being either -1 or 1. Consequently, even after classifying a point correctly, we  are still heavily penalised for it, while with losses like $\mathcal{L}_2$ the closer we are to the correct classification - the smaller the penalty is.
 
Finally, in terms of generalisation capabilities margin-based losses seem to outperform the remaining families. One could argue that this is just a result of lack of regularisation in the rest of the losses, however we underline that all the analysed networks use strong dropout to counter the overfitting problem, and that typical $\mathrm{L}_1$ or $\mathrm{L}_2$ regularisation penalties do not work well in deep networks. 

For CIFAR10 dataset we used a simple convnet, consisting of 3 layers of convolutions, each of size 5x5 and 64 filters, with ReLU activation functions, batch-normalisation and pooling operations in between them (max pooling after first layer and then two average poolings, all 3x3 with stride 2), followed by a single fully connected hidden layer with 128 ReLU neurons, and final softmax layer with 10 neurons.
As one can see in Figure~\ref{fig:mnist_train}, despite completely different architecture than before, we obtain very similar results -- higher order margin losses lead to faster training and significantly stronger models.
Quite surprisingly -- $\mathcal{L}_2$ loss also exhibits similar property. Expectation losses again learn much slower (with the regularised one -- training at the level of log loss and unregularised -- significantly worse).
We would like to underline that this is a very simple architecture, far from the state-of-the art models for CIFAR10, however we wish to avoid using architectures which are heavily overfitted to the log loss. Furthermore, the aim of this paper is not to provide any state-of-the-art models, but rather to characterise effects of loss functions on deep networks. 

As the final interesting result in these experiments, we notice that Cauchy-Schwarz Divergence as the optimisation criterion seems to be a consistently better choice than log loss. It performs equally well or better on both MNIST and CIFAR10 in terms of both learning speed and the final performance. At the same time this information theoretic measure is very rarely used in DL community, and rather exploited in shallow learning (for both classification~\cite{czarnecki2015maximum} and clustering~\cite{principe2000information}).
 
% \section{Noise robusteness}
Now we focus on the impact these losses have on noise robustness of the deep nets.
We start by performing the following experiment on previously trained MNIST classifiers: we add noise sampled from  $\mathcal{N}(0, \epsilon \mathbf{I})$ to each $\mathbf{x}_i$ and observe how quickly (in terms of growing $\epsilon$) network's training accuracy drops (Figure~\ref{fig:mnist_eps}).
\begin{figure}[htb]
 \includegraphics[width=0.925\textwidth]{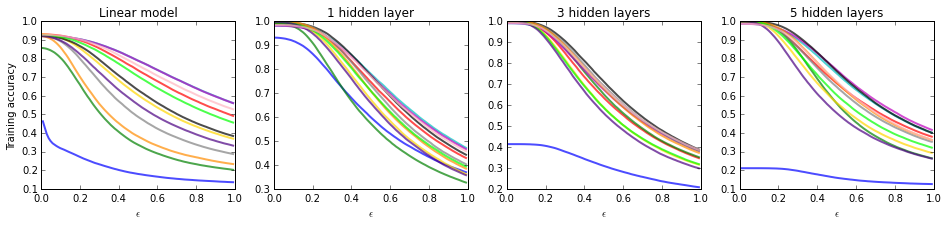}\\
% \hline
 \includegraphics[width=\textwidth]{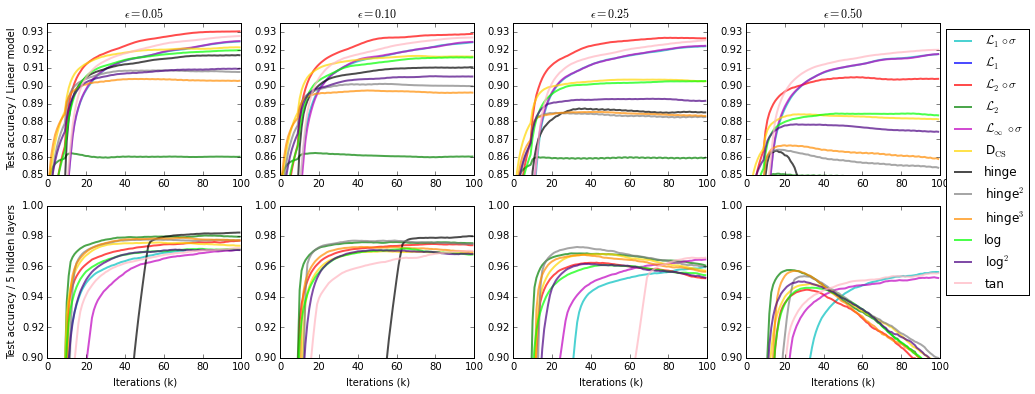}
 \caption{Top row: Training accuracy curves for the MNIST trained models, when presented with training examples with added noise from $\mathcal{N}(0, \epsilon \mathbf{I})$, plotted as a function of $\epsilon$. Middle and bottom rows: Testing accuracy curves for the MNSIT experiment with $\epsilon$ of training labels changed, plotted as a function of training iteration. If $\mathcal{L}_1 \circ \sigma$ is not visible, it is almost perfectly overlapped by $\mathcal{L}_\infty \circ \sigma$.}
 \label{fig:mnist_eps}
\end{figure}
The first crucial observation is that both expectation losses perform very well in terms of input noise robustness. We believe that this is a consequence of what Proposition~\ref{prop:l1exp} showed about their probabilistic interpretation -- that they lead to minimisation of the expected misclassification, which is less biased towards outliers than log loss (or other losses that focus on maximisation of probability of correct labelling of all samples at the same time). For log loss a single heavily misclassified point has an enormous impact on the overall error surface, while for these two losses -- it is minor.
Secondly, margin based losses also perform well on this test, usually slightly worse than the expectation losses, but still better than log loss. This shows that despite no longer maximising the misclassification margin while being used in deep nets -- they still share some characteristics with their linear origins (SVM).
In another, similar experiment, we focus on the generalisation capabilities of the networks trained with increasing amount of label noise in the training set (Figure~\ref{fig:mnist_eps}) and obtain analogous results, showing that robustness to the noise of expectation and margin losses is high for both input and label noise for deep nets, while again -- slightly different results are obtained for linear models, where log loss is more robust than the margin-based ones.
What is even more interesting, a completely non-standard loss function -- \emph{Tanimoto loss} -- performs extremely well on this task. We believe that its exact analysis is one of the important future research directions.

\section{Conclusions}
This paper provides basic analysis of effects the choice of the classification loss function has on deep neural networks training as well as their final characteristics. We believe the obtained results will lead to a wider adoption of various losses in DL work -- where up till now log loss is unquestionable favourite.

In the theoretical section we show that, surprisingly, losses which are believed to be applicable mostly to regression, have a valid probabilistic interpretation when applied to deep network-based classifiers. We also provide theoretical arguments explaining why using them might lead to slower training, which might be one of the reasons DL practitioners have not yet exploited this path.
Our experiments lead to two crucial conclusions. First, that intuitions drawn from linear models rarely transfer to highly-nonlinear deep networks. Second, that depending on the application of the deep model -- losses other than log loss are preferable. In particular, for purely accuracy focused research, squared hinge loss seems to be a better choice at it converges faster as well as provides better performance. It is also more robust to noise in the training set labelling and slightly more robust to noise in the input space. However, if one works with highly noised dataset (both input and output spaces) -- the expectation losses described in detail in this paper -- seem to be the best choice, both from theoretical and empirical perspective. 

At the same time this topic is far from being exhausted, with a large amount of possible paths to follow and questions to be answered. In particular, non-classical loss functions such as Tanimoto loss and Cauchy-Schwarz Divergence are worth further investigation.

\bibliographystyle{plain}
%\bibliography{biblio.bib}

\end{document}